\documentclass[conference]{IEEEtran}
\IEEEoverridecommandlockouts
\usepackage[normalem]{ulem}
\usepackage{cite}
\usepackage{graphicx}
\usepackage{textcomp}
\usepackage{xcolor}
\usepackage[utf8]{inputenc}
\usepackage{subfigure}
\usepackage{enumitem}
\usepackage[T1]{fontenc}% optional T1 font encoding
\usepackage{graphicx,booktabs,multirow,bm}
\newcommand{\TT}{{\mathrm{T}}}
\usepackage{algorithm}
\usepackage{algpseudocode}
\usepackage{array}
\usepackage{tabularx}
\usepackage{amsmath,amssymb,amsfonts,amsthm}
\usepackage{bbm}
\theoremstyle{plain}
\newtheorem{definition}{\bf{Definition}}
\newtheorem{assumption}{\bf{Assumption}}
\newtheorem{theorem}{\bf{Theorem}}

\newtheorem{lemma}{\bf{Lemma}}
\newtheorem{remark}{\bf{Remark}}

\usepackage{xcolor}
\usepackage{hyperref}
\usepackage{url}
\usepackage{cite}
\def\BibTeX{{\rm B\kern-.05em{\sc i\kern-.025em b}\kern-.08em
    T\kern-.1667em\lower.7ex\hbox{E}\kern-.125emX}}
\begin{document}

\title{Stochastic Coded Federated Learning with Convergence and Privacy Guarantees}
% Privacy-Performance Tradeoff in Stochastic Coded Federated Learning}
% \title{Improve both Efficiency and Privacy of Coded Federated Learning}
% \title{Stochastic Coded Federated Learning: Tradeoff between Performance and Privacy}
% Stochastic Coded Computing for Federated Learning with Stragglers}
% Towards better Coded Federated Learning: Improvement on Efficiency and Privacy}
% Fast-Convergent Coded Federated Learning: Tradeoff between Performance and Privacy
% A Privacy-Performance Tradeoff
% \title{Stochastic Coded Federated Learning: Tradeoffs between Privacy, Efficiency, and Performance}

\author{Yuchang~Sun$^{*\ddag}$, Jiawei~Shao$^{*\ddag}$, Songze~Li$^{\ddag\S}$, Yuyi~Mao$^{\dagger}$, and Jun~Zhang$^{\ddag}$,~\IEEEmembership{Fellow,~IEEE}\\
% ,~\IEEEmembership{Fellow,~IEEE}
%\author{Yuchang~Sun$^{*}$, Jiawei~Shao$^{*}$, Songze~Li\\
\IEEEauthorblockA{
Email: \{yuchang.sun, jiawei.shao\}@connect.ust.hk, songzeli@ust.hk, yuyi-eie.mao@polyu.edu.hk, eejzhang@ust.hk}
$^{\ddag}$The Hong Kong University of Science and Technology \\
$^{\S}$The Hong Kong University of Science and Technology (Guangzhou) \\
$^{\dagger}$The Hong Kong Polytechnic University\\
% , Hong Kong, China
%\address{}
%Email: \{yuchang.sun, jiawei.shao\}@connect.ust.hk, songzeli@ust.hk}
% <-this % stops a space
%\thanks{The authors are with the Department of Electronic and Information Engineering, Hong Kong Polytechnic University, Hong Kong (E-mail: jiawei.shao@connect.polyu.hk, \{yuyi-eie.mao, jun-eie.zhang\}@polyu.edu.hk). (The corresponding author is Jun Zhang.)}
\thanks{*Equal contribution.}
}

% \author{\IEEEauthorblockN{Jiawei Shao, Jun Zhang}
% \IEEEauthorblockA{Department of Electronic and Information Engineering \\
% The Hong Kong Polytechnic University, Hong Kong\\
% Email: jiawei.shao@connect.polyu.hk, jun-eie.zhang@polyu.edu.hk}
% }

%\thispagestyle{plain}
%\pagestyle{plain}
\maketitle

\begin{abstract}
%Federated learning (FL) has attracted much attention as a privacy-preserving solution for distributed machine learning, where participants learn a model collaboratively by exchanging model updates with a parameter server without sharing their raw data.
Federated learning (FL) has attracted much attention as a privacy-preserving distributed machine learning framework, where many clients collaboratively train a machine learning model by exchanging model updates with a parameter server instead of sharing their raw data.
Nevertheless, FL training suffers from slow convergence and unstable performance due to stragglers caused by the heterogeneous computational resources of clients and fluctuating communication rates.
% The straggling clients prolong the computation time and degrade the training efficiency.
%the straggling clients prolong the computation time and degrade the training efficiency.
%To mitigate the straggler effect, we propose a training algorithm, namely \emph{stochastic coded federated learning} (SCFL), for linear regression problems.
This paper proposes a coded FL framework to mitigate the straggler issue, namely \emph{stochastic coded federated learning} (SCFL).
%for linear regression problems.
% {\color{red} In particular, each client device privately generates parity training data and shares it with the central server only once at the start of the training phase. The central server can then preemptively perform redundant gradient computations on the coded parity data to compensate for the erased or delayed parameter updates.
% 1. construct a coded dataset.}
% {\color{red}2. perform stochastic gradient descent on the mini-batch data.}
In this framework, each client generates a privacy-preserving coded dataset by adding additive noise to the random linear combination of its local data. 
%Before training starts, the server collects the coded datasets from all the clients to construct a composite coded dataset, which helps to compensate for the straggling effect.
The server collects the coded datasets from all the clients to construct a composite dataset, which helps to compensate for the straggling effect.
% and the extra gradients calculated by the server compensates for straggling gradients.
%, and perform gradient descent in the coded domain
In the training process, the server as well as clients perform mini-batch stochastic gradient descent (SGD), and the server adds a make-up term in model aggregation to obtain unbiased gradient estimates.
%adopt an efficient method, mini-batch stochastic gradient descent (SGD) for model updating. % in contrast to existing coded FL methods that require full-batch gradients,
We characterize the privacy guarantee by the mutual information differential privacy (MI-DP) and analyze the convergence performance in federated learning.
Besides, we demonstrate a privacy-performance tradeoff of the proposed SCFL method by analyzing the influence of the privacy constraint on the convergence rate.
%Moreover, the privacy leakage in coded data sharing is characterized by the mutual information differential privacy (MI-DP).
%{\color{blue}\sout{, where we can control the privacy budget by adjusting the added noise level}}.
%We theoretically analyze the convergence rate and show a privacy-performance tradeoff.
% We analyze the convergence rate and characterize the privacy leakage via the mutual information differential privacy (MI-DP) of SCFL, which demonstrates a tradeoff between privacy and efficiency.
% Particularly, we leverage the mutual information differential privacy (MI-DP) to characterize the privacy leakage of the coded datasets.
%In the proposed scheme, clients utilize random projection and additive noise to generate the coded dataset locally and transmit them to the server.
%In the training, the server and clients compute the stochastic gradients instead of full-batch gradients, which effectively saves the computation cost.
%Additionally, we analyze the convergence of SCFL and discuss the effects of various system parameters.
%Besides, we leverage the mutual information differential privacy (MI-DP) to measure the privacy leakage incurred by uploading the coded dataset.
%Based on analytical results, we show a tradeoff between model performance and privacy budget in SCFL.
Finally, numerical experiments corroborate our analysis and show the benefits of SCFL in achieving fast convergence while preserving data privacy.
%, and demonstrate the privacy-performance tradeoff. compared with existing CFL methods
\end{abstract}


\section{Introduction}

% ML - FL - CFL
The recent development of deep learning (DL) has led to main breakthroughs in various domains, including healthcare \cite{healthcare}, autonomous vehicles \cite{autodriving}, and the Internet of Things (IoT) \cite{iot}.
These applications in turn lead to an unprecedented volume of data generated at the wireless network edge by massive end devices. 
To utilize these data for DL model training, the traditional approach is to directly upload them to a cloud server.
However, such a centralized approach may raise severe privacy concerns, as the local data collected by devices usually contain private and sensitive information \cite{meneghello2019iot}.

To resolve this issue, federated learning (FL) \cite{fedavg} was proposed by Google to collaboratively learn a global model without sharing local data. 
A canonical FL system consists of a centralized parameter server and a large number of clients (e.g., IoT and mobile devices). 
In the training process, the clients perform training locally on their private data and upload the model updates to the server. After receiving the model updates from the clients, the server aggregates them into a new global model by weighted averaging. 
One of the main challenges in FL is the straggler effect, where a small number of significantly slower devices or unreliable wireless links may drastically prolong the training time.
Since the client data in FL are non-independent and identically distributed (non-IID) \cite{PartialWorkerLogistic}, simply ignoring the stragglers may degrade the training efficiency and model performance.
Motivated by the \emph{coded computing} techniques \cite{li2020coded,yu2019lagrange,codingnew1,codingnew2,codingnew3}, \emph{coded federated learning} (CFL) \cite{CFL,CFL_journal,anand2021differentially} has been recently proposed to mitigate the stragglers in federated linear regression by constructing the coded datasets.
% in federated linear regression
In CFL \cite{CFL,CFL_journal}, each client generates a coded dataset by applying random linear projection on their weighted local data, which is uploaded to a centralized server before training starts.
During training, the server uses these coded datasets to compute the coded gradients in order to compensate for the missing gradients from the straggling clients.
% {
% \color{red}
However, as the CFL-FB \cite{CFL} framework performs gradient descent in a full-batch (FB) manner, the training process is computationally expensive.
%  {\color{blue}based on all training samples \sout{
A variant of CFL-FB, namely CodedFedL, was investigated in \cite{CFL_journal}, which adopts the mini-batch stochastic gradient descent (SGD) \cite{dekel2012optimal} algorithm to improve the training efficiency.
Although the convergence of CodedFedL was analyzed in \cite{CFL_journal}, it relies on simplified assumptions by neglecting the variance from mini-batch sampling.
%approximated the coded gradient by the full-batch gradient and did not consider the mini-batch variance.
%\sout{Besides, the authors of [11] utilized the mutual-information differential privacy (MI-DP) [14] to measure the privacy leakage in coded data sharing,  but the influence of the privacy constraint on the convergence performance of CFL has not been well studied.} 
Moreover, the interplay between privacy leakage in coded data sharing and convergence performance of CFL is not well understood.
%the influence of the privacy constraint on the convergence performance has not been well studied.
% Besiedes, the authors of \cite{CFL_journal} utilized the mutual-information differential privacy (MI-DP) \cite{MI-DP} to measure the privacy leakage in coded data sharing,  but the influence of the privacy constraint on the convergence performance of CFL has not been well studied.
% The only convergence analysis provided in \cite{CFL_journal} simplifies the coding effect as a pre-defined assumption and does not consider the mini-batch variance,
%However, the privacy leakage in coded data sharing has not been discussed in \cite{CFL}. 
%Although the authors of \cite{CFL_journal} utilized the mutual-information differential privacy (MI-DP) \cite{MI-DP} to characterize the privacy leakage, the influence of the privacy constraint on the convergence performance of CFL has not been well studied.
% \sout{The aggregated gradient on the server is computed on the data samples across all devices, which, however, is computationally expensive.}
%Besides, a recent work,  
Most recently, a differentially private coded federated learning (DP-CFL) scheme that adds Gaussian noise to the coded data was proposed in \cite{anand2021differentially} for a better privacy guarantee.
% \sout{that flexibly adds the Gaussian noise to the coded dataset for different privacy constraints.}
%This work characterized the distortion of gradient estimates incurred from the additive noise.
%it also studies the variance of gradient estimates induced from the noise.
Nevertheless, DP-CFL restricts the gradient computation on the server once the coded datasets are collected, which fails to exploit the local computational resources at the clients for fast distributed training.

%without further communication with clients}. Thus, the computation resources of clients cannot be fully utilized, which prolongs the training process.
% \sout{Mini-batch stochastic gradient descent (SGD) is commonly utilized to achieve efficient model training, which has the potential to be combined with CFL.
% Although the experiment in  evaluated a mini-batch SGD variant of CFL, where each client only encodes a portion of local dataset in each epoch. the detailed design and theoretical analysis were not considered.}
% In previous works, the theoretical analysis of CFL in terms of the privacy and convergence rate has not been well studied.
% The only convergence analysis provided in \cite{CFL_journal} simplifies the coding effect as a pre-defined assumption and does not consider the mini-batch variance, which is less practical.
% }

% {\color{red}
In this paper, we propose a \emph{stochastic coded federated learning} (SCFL) framework for efficient federated linear regression.
Specifically, each client generates a privacy-preserving coded dataset by adding Gaussian noise to the random linear combination of its local data. 
The server collects the coded datasets from all the clients to construct a composite dataset, which helps to compensate for the straggling effect.
%The clients utilize random projection and the additive Gaussian noise to generate the coded dataset.
% {\color{red}
% \sout{In the training process, the server and clients perform mini-batch SGD for model updating. 
% % We design a novel aggregation scheme to ensure that an unbiased estimation of gradient is achieved in training.
% % {\color{blue}
% To achieve an unbiased estimation of gradient, we add a make-up term in model aggregation.}
% }
In the training process, the server as well as clients compute the stochastic gradients on a batch of samples, and the server adds a make-up term in model aggregation to obtain unbiased gradient estimates.
% }
We characterize the privacy guarantee in coded data sharing using the mutual information differential privacy (MI-DP) method \cite{MI-DP} and analyze the convergence performance of SCFL. Besides, we theoretically demonstrate the privacy-performance tradeoff of the proposed SCFL framework by analyzing the influence of the privacy constraint on the convergence rate.
% {\color{blue}
% We prove the convergence for SCFL considering the underlying distribution of both encoding matrices and mini-batch sampling matrices.
% Moreover, we characterize the privacy guarantee in coded data sharing using the MI-DP method \cite{MI-DP} and analyze the privacy-performance tradeoff in SCFL.
% \sout{We characterize the privacy guarantee in coded data sharing using the MI-DP method and analyze the convergence performance for SCFL. Besides, we theoretically demonstrate the privacy-performance tradeoff of the proposed SCFL method by analyzing the influence of the privacy constraint on the convergence rate.}
% }
% , which can be flexibly adjusted by varying the additive noise level \cite{DPDL}.
%Using the mutual information differential privacy (MI-DP) \cite{MI-DP}, we theoretically characterize the privacy budget, which can be flexibly achieved by adjusting the additive noise level \cite{DPDL}.
% \sout{Besides, we provide the complete convergence analysis of the proposed SCFL and discuss the effects of various parameters.
% Based on the theoretical results, we show the tradeoff between the privacy guarantee and model performance.}
% is discussed by analyzing the convergence rate of the proposed SCFL method under different privacy constraints.
Numerical experiments demonstrate such a tradeoff and show the benefits of SCFL in achieving fast convergence while preserving data privacy.

\section{System Model}\label{sec:problem}

\subsection{Federated Learning for Linear Regression}

We consider an FL system with a centralized server and $n$ clients. They collaborate to train a model $\mathbf{W}\in\mathbb{R}^{d\times o}$, where $d$ and $o$ are respectively the input and output dimensions.
We focus on the linear regression problem over the training dataset $(\mathbf{X},\mathbf{Y})$, where $\mathbf{X}\in\mathbb{R}^{m\times d}$ concatenates the $d$-dimensional features of $m$ data samples, and $\mathbf{Y}\in\mathbb{R}^{m\times o}$ represents the corresponding labels.
%The objective is to minimize the expected loss associated with every data sample $(\boldsymbol{x},\boldsymbol{y})$ as follows:
We formulate the following empirical risk minimization problem to optimize $\mathbf{W}$:
\begin{equation}
    \min_{\mathbf{W}} f(\mathbf{W}) =  \frac{1}{2}  \|\mathbf{X}\mathbf{W} - \mathbf{Y} \|_\mathrm{F}^{2}.
\label{eq:loss}
\end{equation}
% \begin{equation}
% \begin{aligned}
%     \min_{\mathbf{W}} f(\mathbf{W}) = &
%     \mathbb{E}_{(\boldsymbol{x},\boldsymbol{y}) \sim \mathcal{D}} \left\{ \frac{1}{2}  \|\boldsymbol{x}\mathbf{W} - \boldsymbol{y} \|_{2}^2 \right\}, \\
%     \simeq &  \frac{1}{2m}  \|\mathbf{X}\mathbf{W} - \mathbf{Y} \|_\mathrm{F}^{2}.
% \end{aligned}
% \end{equation}
% {\color{blue}where each data sample (\boldsymbol{x}, \boldsymbol{y}) is generated from an unknown distribution \mathcal{D}.}
% Note that if the training dataset $(\mathbf{X},\mathbf{Y})$ is available at the server, the optimal solution of \eqref{eq:loss} can be obtained by setting the gradient of $f\left(\mathbf{W}\right)$ to zero, i.e., $\nabla f(\mathbf{W}) \triangleq \mathbf{X}^\TT (\mathbf{X}\mathbf{W}-\mathbf{Y})=0$.
Since the data samples are stored locally, each client $i$ can only access $l_i$ local data samples, denoted as $\mathbf{X}^{(i)}=\mathbf{Z}^{(i)}\mathbf{X}, \mathbf{Y}^{(i)}=\mathbf{Z}^{(i)}\mathbf{Y}$, where the matrix $\mathbf{Z}^{(i)} \triangleq [ z_{j,j^\prime}^{(i)} ] \in \mathbb{R}^{l_i \times m}$ represents the data availability on client $i$ and $|\mathbf{Z}^{(i)}|_0=l_{i}$.
Specifically, $z_{j^\prime,j}^{(i)}=1$ if the $j$-th data sample in the dataset corresponds to the $j^\prime$-th data sample at client $i$; otherwise $z_{j^\prime,j}^{(i)}=0$.
% $\sum_{i=1}^n \mathbf{Z}^{(i)} = \mathbf{I}$
We assume that the local datasets are disjoint across the clients.
Accordingly, we have $\mathbf{X} = [\mathbf{X}^{(1)\TT}, \mathbf{X}^{(2)\TT},\dots,\mathbf{X}^{(n)\TT}]^{\TT}$, $\mathbf{Y} = [\mathbf{Y}^{(1)\TT}, \mathbf{Y}^{(2)\TT},\dots,\mathbf{Y}^{(n)\TT}]^{\TT}$, and $m=\sum_{i=1}^{n} l_{i}$.

In canonical FL systems, the centralized server updates the global model based on the uploaded gradients from the clients.
However, the training efficiency is critically affected by the delays in computing and communicating the gradients. Particularly, the server may need to wait for a few straggling clients in each training epoch.
To characterize this phenomenon, we describe the computation and communication models in the next subsection.

\subsection{Computation and Communication Models}
\label{subsec:delay}
% {\color{red}
% {\color{blue}
Following previous works on CFL \cite{CFL,CFL_journal}, we consider clients with different processing rates and link constraints.
In each epoch, the time needed for client $i$ to complete local training and gradient exchange with the server is expressed as $t_i(b_i) = t_{c_i}(b_i) + t_{u_i} + t_{d}$.
Here, $t_{c_i}(b_i) = \frac{b_i N_\text{MAC}}{\text{MACR}_{i}}$ denotes the time for computing the gradient on a batch of $b_i$ data samples, where $\text{MACR}_{i}$ denotes the Multiply-Accumulate (MAC) rate, and $N_\text{MAC}$ is the number of MAC operations required for processing one data sample.
We assume the server follows the same computation model as the clients but with a higher MAC rate.
% \sout{In addition to local computation time, each client $i$ needs $t_{d_i}$ time to download model from the server and $t_{u_i}$ time to upload the model update.
% We assume the server broadcasts the model \sout{parameter} to all the clients with the same rate $r_{\text {downlink}}$.}
Besides, $t_{u_i}$ denotes the gradient uploading time for client $i$, and $t_{d}$ is the model downloading time.
%{\color{black}
As the uplink transmission exhibits strong stochastic fluctuations in the link quality, we model the uplink data transmission of the $i$-th client by a data rate $r_{\text{uplink}, i}$ and a link erasure probability $p_{l_i}$ \cite{lte-book}, i.e., the expected number of transmissions attempt $N_i$ required before successful communication from the $i$-th client to the server follows a geometric distribution.
In addition, we assume the server broadcasts the model to all the clients reliably with rate $r_{\text {downlink}}$, attributed to more capable downlink radio resources.
Therefore, the arrival probability that the model update of client $i$ is received by the server within time $T$ is given by $p_i(T,b_i) \triangleq \mathrm{Prob} \{ t_i(b_i) \leq T \}$.
% }
%\sout{Besides, we assume the server follows the same computation model as the clients but with a higher MAC rate.}
% }
%$\tau_i =\frac{x}{r_i}$ is the time to send a packet of size $x$ bits.
%Thus, the averaged total delay is $E[t_i(b_i)] = b_i a_i + \frac{2 \tau_i}{1-p_i}$.
%\begin{eqnarray}
%\label{commModel}
%Pr\{N_i=t\} = p^{t-1}_i(1-p_i), \,\, t = 1,2,3,...
%\end{eqnarray}
%\noindent It is a typical practice to dynamically adapt the data rate $r_i$ with respect to the changing quality of the wireless link while maintaining a constant erasure probability $p$ during the entire gradient computation. 
% Furthermore, without loss of generality, we can assume uplink and downlink channel conditions are reciprocal. Thus, Downlink and uplink communication delays are random variables given as follows:
% \begin{eqnarray}
% \label{commTime}
% t_{d_i} =  N_i \tau_i, \\
% t_{u_i} = N_i \tau_i,
% \end{eqnarray}
% The total time taken by the $i$-th device in each round is $t_i(b_i) = t_{c_i}(b_i) + t_{d_i} + t_{u_i}$.
%to receive the updated model, compute and successfully communicate the partial gradient to the master device is
% \begin{eqnarray}
% \label{totDelay}
% t_i(b_i) = t_{c_i}(b_i) + t_{d_i} + t_{u_i}.
% \end{eqnarray}
%Besides, 
% \begin{eqnarray}
% \label{avgDelay}
% E[t_i(b_i)] = b_i a_i + \frac{2 \tau_i}{1-p}.
% \end{eqnarray}

% @article{lte,
%   title={LTE-advanced},
%   author={Wannstrom, Jeanette},
%   journal={Third Generation Partnership Project (3GPP)},
%   year={2013}
% }

\section{Stochastic Coded Federated Learning}\label{sec:method}

In this section, we introduce the proposed Stochastic Coded Federated Learning (SCFL) framework.

\subsection{Coded Data Preparation}
Client $i$ generates the coded data locally using the random projection matrix (denoted as $\mathbf{G}_{i} \in \mathbb{R}^{c \times l_i}$) and the additive Gaussian noise (denoted as $\sigma \mathbf{N}_{i} \in \mathbb{R}^{c \times d}$), where $c$ is the amount of generated coded data, and $ \sigma \geq 0$ controls the noise level. 
In particular, each entry of $\mathbf{G}_{i}$ and $\mathbf{N}_{i}$ is independently sampled from the standard normal distribution $\mathcal{N}(0,1)$. 
The coded dataset is computed according to $\tilde{\mathbf{X}}^{(i)}=\mathbf{G}_{i} \mathbf{X}^{(i)}+ \sigma \mathbf{N}_{i}$ and $\tilde{\mathbf{Y}}^{(i)}=\mathbf{G}_{i} \mathbf{Y}^{(i)}$.
As a result, the coded dataset on the server can be expressed as:
\begin{equation}
    \tilde{\mathbf{X}}=\mathbf{G} \mathbf{X} + \sigma \mathbf{N}, \quad 
    \tilde{\mathbf{Y}}=\mathbf{G} \mathbf{Y},
\label{equ:encoding}
\end{equation}
where $\tilde{\mathbf{X}} \triangleq \sum_{i=1}^n \tilde{\mathbf{X}}^{(i)}$, $\mathbf{G} \triangleq [\mathbf{G}_1, \mathbf{G}_2,\dots,\mathbf{G}_n]$, $\mathbf{N} \triangleq \sum_{i=1}^n \mathbf{N}_i$, and $\tilde{\mathbf{Y}} \triangleq \sum_{i=1}^n \tilde{\mathbf{Y}}^{(i)}$.
Note that the server only has access to the coded dataset $(\tilde{\mathbf{X}}, \tilde{\mathbf{Y}})$ without knowing matrices $\mathbf{G}$, $\mathbf{N}$, $\mathbf{X}$ and $\mathbf{Y}$.
% are unknown at the server.
% \sout{cannot decode the original data samples without $\mathbf{G}$ and $\mathbf{N}$.}
It is noteworthy that the above operations only occur once before training starts and thus incur negligible communication overhead.

\subsection{Stochastic Gradient Computation}
% NOTE: The sampling is similar with <Stochastic Client Selection for Federated Learning with Volatile Clients>: multinomialNR(p_t/k,k)
%{\color{blue}\sout{In contrast to most CFL methods that require full-batch gradients,}}
In the $r$-th training epoch of SCFL, to compensate the potential straggling clients, the server samples a random batch of the coded data with size $b_\mathrm{s}$ and computes the gradient $g_\mathrm{s} (\mathbf{W}^{(r)})$ on $\mathbf{\hat{X}}^{(r)\TT}_\mathrm{s} \!=\! \mathbf{S}_\mathrm{s}^{(r)} \tilde{\mathbf{X}}$ and $\mathbf{\hat{Y}}^{(r)\TT}_\mathrm{s} \!=\! \mathbf{S}_\mathrm{s}^{(r)} \tilde{\mathbf{Y}}$, where $\mathbf{S}_\mathrm{s}^{(r)} \!\triangleq\! \mathrm{diag}(s_1^{(r)},s_2^{(r)},\dots,s_c^{(r)}) \in \mathbb{R}^{c\times c}$ is a diagonal matrix denoting the result of uniform sampling without replacement, and each diagonal entry $s_j^{(r)}$ follows a Bernoulli distribution, i.e., $s_j^{(r)}\sim \mathrm{Bernoulli}(\frac{b_\mathrm{s}}{c})$.
Specifically, $s_j^{(r)}\!=\!1$ if the $j$-th coded data sample is selected in the $r$-th epoch; otherwise $s_j^{(r)}\!=\!0$.
%\sout{Despite mini-batch sampling effectively reduces the computation latency, it incurs additional variance. This will be discussed in Section \ref{sec:theory}.}
%the stochastic gradient is an estimation of the true gradient \cite{dekel2012optimal} ,which incurs additional variance, as will be explained in Section \ref{sec:theory}.
Similarly, client $i$ samples a batch of its local data with size $b_i$ to compute the gradient $g_i(\mathbf{W}^{(r)})$.
Denote the sampling matrix as $\mathbf{S}^{(i,r)} \!\triangleq\! \mathrm{diag}(s_1^{(i,r)},s_2^{(i,r)},\dots,s_{l_i}^{(i,r)}) \in \mathbb{R}^{l_i\times l_i}$, where each diagonal element $s_j^{(i,r)}$ follows a Bernoulli distribution, i.e., $s_j^{(i,r)}\!\sim\! \mathrm{Bernoulli}(\frac{b_i}{l_i})$.
The mini-batch data sampled in this epoch is expressed as $\mathbf{\hat{X}}^{(i,r)}\!=\!\mathbf{S}^{(i,r)} \mathbf{X}^{(i)}$ and $\mathbf{\hat{Y}}^{(i,r)}\!=\!\mathbf{S}^{(i,r)} \mathbf{Y}^{(i)}$.
After stochastic gradient computation, we have $g_\mathrm{s} (\mathbf{W}^{(r)}) \!=\! \frac{1}{b_\mathrm{s}} \mathbf{\hat{X}}^{(r)\TT}_\mathrm{s} (\mathbf{\hat{X}}^{(r)}_\mathrm{s} \mathbf{W}^{(r)} - \mathbf{\hat{Y}}^{(r)}_\mathrm{s})$ and $g_i (\mathbf{W}^{(r)}) \!=\! \frac{l_i}{b_i}\mathbf{\hat{X}}^{(i,r)\TT} (\mathbf{\hat{X}}^{(i,r)} \mathbf{W}^{(r)} - \mathbf{\hat{Y}}^{(i,r)})$.
In every epoch, the server waits for a duration of time $T$, and aggregates the received stochastic gradients as follows:
% \vspace{-5pt}
\begin{equation}
% \begin{split}
    g(\mathbf{W}^{(r)})
    = \frac{1}{2} \! \bigg[ \sum_{i=1}^n \hat{g}_i (\mathbf{W}^{(r)})
    + g_\mathrm{s} (\mathbf{W}^{(r)}) + g_\mathrm{o}(\mathbf{W}^{(r)}) \bigg],
% \end{split}
\label{eq:acc}
\end{equation}
% {\color{red}
% where $ \hat{g}_i (\mathbf{W}^{(r)}) \triangleq \frac{g_i (\mathbf{W}^{(r)})}{p_i(T,b_i)} \mathbbm{1} \{ t_i(b_i)\leq T \}$ denotes the arrived gradient from the $i$-th client, and $\mathbbm{1} \{ t_i(b_i)\leq T \}$ represents the arrival status. 
% The weight $\frac{1}{p_i(T,b_i)}$ in $ \hat{g}_i (\mathbf{W}^{(r)})$ achieves the unbiased estimation of client's gradient, i.e., $\mathbb{E}[\hat{g}_{i}(\mathbf{W}^{(r)})] = g_{i}(\mathbf{W}^{(r)})$.
% Besides, $g_\mathrm{o}(\mathbf{W}^{(r)}) \triangleq - \sigma^2 \mathbf{W}^{(r)}$ is a make-up term to compensate for the bias induced from the coded datasets, which ensures that the aggregated gradient $g(\mathbf{W}^{(r)})$ is an unbiased gradient of the empirical risk function in (\ref{eq:loss}), i.e., $\nabla f(\mathbf{W}^{(r)}) \triangleq \mathbf{X}^\TT (\mathbf{X}\mathbf{W}-\mathbf{Y})$.
% In the training process, the server updates the global model according to $\mathbf{W}^{(r+1)}=\mathbf{W}^{(r)}-\eta_r g(\mathbf{W}^{(r)})$ where the $\eta_r$ denotes the learning rate. After that, the model is transmitted to the clients for next training epoch.
% }
% {\color{blue}
where $ \hat{g}_i (\mathbf{W}^{(r)}) \triangleq \frac{g_i (\mathbf{W}^{(r)})}{p_i(T,b_i)} \mathbbm{1} \{ t_i(b_i)\leq T \}$ denotes the arrived gradient from client $i$, $\mathbbm{1} \{ t_i(b_i)\leq T \}$ represents the arrival status, and $g_\mathrm{o}(\mathbf{W}^{(r)}) \triangleq - n \sigma^2 \mathbf{W}^{(r)}$.
We assign a higher weight $\frac{1}{p_i(T,b_i)}$ to the clients with lower arrival probabilities, such that $\mathbb{E}[\hat{g}_{i}(\mathbf{W}^{(r)})] = g_{i}(\mathbf{W}^{(r)})$.
Besides, as adding Gaussian noise leads to a bias in the gradient estimation, the make-up term $g_\mathrm{o}(\mathbf{W}^{(r)})$ erases the bias incurred by data coding.
Thus, such an aggregation scheme guarantees that the aggregated gradient $g(\mathbf{W}^{(r)})$ is an unbiased gradient of the empirical risk function in (\ref{eq:loss}), i.e., $\nabla f(\mathbf{W}^{(r)}) \triangleq \mathbf{X}^\TT (\mathbf{X}\mathbf{W}-\mathbf{Y})$.
In the training process, the server updates the global model according to $\mathbf{W}^{(r+1)}=\mathbf{W}^{(r)}-\eta_r g(\mathbf{W}^{(r)})$ where $\eta_r$ denotes the learning rate in the $r$-th epoch.
% \sout{After that, the model is transmitted to the clients for next training epoch.}
After that, the server transmits the model to the clients and maintains a local copy. We adopt the average of these models over the entire training process as the learned model, i.e., $\frac{1}{R}\sum_{r=1}^R \mathbf{W}^{(r)}$, where $R$ is the number of training epochs.
% }

\section{Theoretical Analysis}\label{sec:theory}

\subsection{Convergence Analysis}

% We begin with the formal assumptions for aforementioned $\mathbf{G}$, $\mathbf{N}$, $\mathbf{S}_\mathrm{s}^{(r)}$ and $\mathbf{S}^{(i,r)}$ as follows.
% \begin{assumption}\label{ass:matrix}
% \;
% \begin{itemize}
%     \item Each entry of $\mathbf{G}$ and $\mathbf{N}$ is IID sampled from the standard Gaussian distribution, i.e., $g_{i,j}\sim \mathcal{N}(0,1)$, and $n_{i,j}\sim \mathcal{N}(0,1)$.
%     % \item Each entry of $\mathbf{N}\in\mathbb{R}^{c\times m}$ is IID sampled from the standard Gaussian distribution, i.e., $g_{ij}\sim \mathcal{N}(0,1)$. 
%     \item Each diagonal element of the diagonal matix $\mathbf{S}_\mathrm{s}^{(r)} \in\mathbb{R}^{c\times c}$ is IID sampled from the Bernoulli distribution, i.e., $s_j\sim \mathrm{Bernoulli}(\frac{b_\mathrm{s}}{c})$.
%     \item Each diagonal element of the diagonal matrix $\mathbf{S}^{(i,r)} \in\mathbb{R}^{m\times m}$ is IID sampled from the Bernoulli distribution, i.e., $s_j\sim \mathrm{Bernoulli}(\frac{b_i}{l_i}),\forall i$.
% \end{itemize}
% \end{assumption}

% Denote $\mathbf{X}_i \triangleq [\bm{x}_1^{(i)},\bm{x}_2^{(i)},\dots,\bm{x}_{l_i}^{(i)}]$
%To facilitate the convergence analysis, we make the following assumptions \cite{CFL_journal,bottou2018optimization}.
We first present the following assumptions \cite{CFL_journal,Privacy_utility_tradeoff} to facilitate the convergence analysis:
\begin{assumption}
\label{ass:max_entry}
The maximum absolute value of entries in $\mathbf{X}$ is upper bounded by 1.
\end{assumption}
\begin{assumption}\label{ass:norm}
There exist constants $\{\alpha_i\}$'s, $\{\zeta_i\}$'s, $\{\kappa_i\}$'s, and $\phi$ such that $\alpha_i^2 \leq \left\| \mathbf{X}^{(i)} \right\|_\mathrm{F}^2 \leq \zeta_i^2$, $\left\| \mathbf{X}^{(i)} \mathbf{W}^{(r)} - \mathbf{Y}^{(i)} \right\|_\mathrm{F}^2 \leq \kappa_i^2$, and $\left\| \mathbf{W}^{(r)} \right\|_\mathrm{F}^2 \leq \phi^2$.
\end{assumption}
% Besides, we define $\alpha \triangleq \sum_{i=1}^n \alpha_i^2$, $\zeta \triangleq \sum_{i=1}^n \zeta_i^2$, and $\kappa \triangleq \sum_{i=1}^n \kappa_i^2$.  
% Accordingly, the global loss function is $\alpha$-strxongly convex and $\zeta$-smooth. 

% We rewrite the gradients of server and client $i$ in the compact form as follows:
% \begin{align}
%     g_\mathrm{s} (\mathbf{W}^{(r)})
%     &= \frac{1}{b_\mathrm{s}} \mathbf{\hat{X}}^{(r)\TT}_\mathrm{s} (\mathbf{\hat{X}}^{(r)}_\mathrm{s} \mathbf{W}^{(r)} - \mathbf{\hat{Y}}^{(r)}_\mathrm{s}), \\
%     g_i (\mathbf{W}^{(r)})
%     &= \frac{l_i}{b_i}\mathbf{\Tilde{X}}^{(i,r)\TT} (\mathbf{\Tilde{X}}^{(i,r)} \mathbf{W}^{(r)} - \mathbf{\hat{Y}}^{(i,r)}), \\
%     \nabla f(\mathbf{W}^{(r)})
%     &= \mathbf{X}^\TT (\mathbf{X} \mathbf{W}^{(r)} - \mathbf{Y}).
% \end{align}
We are now to elaborate that the aggregated stochastic gradient in \eqref{eq:acc} is an unbiased estimate (i.e., Lemma \ref{lem-2}) of the gradient $\nabla f(\mathbf{W}^{(r)})$ with the bounded variance (i.e., Lemma \ref{lem-3}). We first derive some important properties of the predefined matrices in the following lemma.
% \sout{over all the training samples}

\begin{lemma}\label{lem-1}
Matrices $\mathbf{G}$, $\mathbf{N}$, $\mathbf{S}_\mathrm{s}^{(r)}$, and $\{\mathbf{S}^{(i,r)}\}$'s have the following properties:
\begin{itemize}[leftmargin=1em]
    \item $\mathbb{E}[ \|\frac{1}{c}\mathbf{G}^\TT\mathbf{G}\|_\mathrm{F} ] = \mathbf{I}_m$ and $\mathbb{E} [\left\| \frac{1}{c}\mathbf{G}^\TT\mathbf{G} - \mathbf{I}_m \right\|_\mathrm{F}^2 ] \!=\! \frac{m+m^2}{c}$.
    \item $\mathbb{E}[ \|\frac{1}{c}\mathbf{N}^\TT\mathbf{N}\|_\mathrm{F} ] \!=\! n \mathbf{I}_d$ and $\mathbb{E} [\left\| \frac{1}{c}\mathbf{N}^\TT\mathbf{N} - n \mathbf{I}_d \right\|_\mathrm{F}^2 ] \!=\! \frac{(d+d^2)n}{c}$.
    \item $\mathbb{E}[ \frac{c}{b_\mathrm{s}} {\mathbf{S}_\mathrm{s}^{(r)\TT}} \mathbf{S}_\mathrm{s}^{(r)} ] \!=\! \mathbf{I}_c$ and $\mathbb{E} [ \| \frac{c}{b_\mathrm{s}} {\mathbf{S}_\mathrm{s}^{(r)\TT}}\mathbf{S}_\mathrm{s}^{(r)} - \mathbf{I}_c \|_\mathrm{F}^2 ]=\frac{c(c-b_\mathrm{s})}{b_\mathrm{s}}$.
    \item $\mathbb{E}[ \frac{l_i}{b_i} {\mathbf{S}^{(i,r)\TT}}\mathbf{S}^{(i,r)} ] \!=\! \mathbf{I}_{l_i}$ and $\mathbb{E} [ \| \frac{l_i}{b_i}{\mathbf{S}^{(i,r)\TT}} \mathbf{S}^{(i,r)} - \mathbf{I}_{l_i} \|_\mathrm{F}^2 ] \!=\! \frac{l_i(l_i-b_i)}{b_i}$.
\end{itemize}
\end{lemma}
\begin{proof}
According to the definition of $\mathbf{G}$, $\frac{1}{c}\mathbf{G}^\TT\mathbf{G}$ follows the Wishart distribution, i.e., $\frac{1}{c}\mathbf{G}^\TT\mathbf{G} \sim \mathcal{W}(c,\mathbf{I}_m)$. 
Besides, since matrix $\mathbf{S}_\mathrm{s}^{(r)}$ is symmetric and diagonal, each entry of $\frac{c}{b_\mathrm{s}} {\mathbf{S}_\mathrm{s}^{(r)\TT}} \mathbf{S}_\mathrm{s}^{(r)}$ has a unit mean. 
{\color{black}The proofs for $\mathbf{N}$ and $\mathbf{S}^{(i,r)}$ can be obtained similarly, which are omitted for brevity.}
% In addition, ${\mathbf{Z}^{(i)\TT}}\mathbf{Z}^{(i)} = \mathbf{Z}^{(i)}$ stands for the data status of client $i$, by summing up which we obtain the indexes for all data samples.
\end{proof}

\begin{lemma}\label{lem-2}
The aggregated gradient in \eqref{eq:acc} is an unbiased estimate of the global gradient, i.e., $\mathbb{E}[g(\mathbf{W}^{(r)})]=\nabla f(\mathbf{W}^{(r)})$.
\end{lemma}
% \begin{proofsketch}
% The proof follows that $\frac{1}{c}\mathbf{G}^\TT\mathbf{G}$ obeys Wishart distribution and each entry of $\frac{l_i}{b_i} {\mathbf{S}^{(i,r)\TT}}\mathbf{S}^{(i,r)}$ obeys the binomial distribution. Please refer to Appendix \ref{proof-lem-2} for detailed proof.
% \end{proofsketch}
\begin{proof}
The result directly follows the properties derived in Lemma \ref{lem-1}.
\end{proof}

\begin{lemma}\label{lem-3}
The variance of stochastic gradients is bounded as follows:
\vspace{-5pt}
\begin{equation}
% \begin{split}
% \setlength{\abovedisplayskip}{0pt}
    \mathbb{E} \left[\left\| g(\mathbf{W}^{(r)}) - \nabla f(\mathbf{W}^{(r)}) \right\|_\mathrm{F}^2 \right] 
    \leq \rho,
    % \leq h(c, \sigma, b_\mathrm{s},\{b_i\}, T),
% \end{split}
\end{equation}
where $\rho \!  \triangleq\! \frac{c-b_\mathrm{s}}{4c b_\mathrm{s}} \zeta\kappa \!+\! \frac{1}{c} (m+m^2) \zeta\kappa + \frac{1}{c} (d+d^2)n \sigma^4 \phi^2 + \frac{dmn \sigma^2}{c^2} (\zeta\phi^2 + \kappa) + \frac{1}{2} \sum_{i=1}^{n} \frac{1 - p_i(T,b_i)}{p_i(T,b_i)} \zeta_i^2\kappa_i^2 + \frac{1}{2} \sum_{i=1}^n \frac{l_i(l_i-b_i)}{b_i} \zeta_i^2\kappa_i^2$, $\alpha \!\triangleq\! \sum_{i=1}^n \alpha_i^2$, $\zeta \!\triangleq\! \sum_{i=1}^n \zeta_i^2$, and $\kappa \!\triangleq\! \sum_{i=1}^n \kappa_i^2$.  
\end{lemma}
% \begin{proof}
% Please refer to Appendix \ref{proof-lem-3}.
% \end{proof}
\begin{proof}
% Due to the independence variances caused by the server and clients, we respectively provide upper bounds for $\mathbb{E} \big[\left\| g_\mathrm{s}(\mathbf{W}^{(r)}) + g_\mathrm{o}(\mathbf{W}^{(r)}) - \nabla f(\mathbf{W}^{(r)}) \right\|_\mathrm{F}^2 \big]$ and $\mathbb{E} \big[\left\| \sum_{i=1}^{n} \hat{g}_i (\mathbf{W}^{(r)}) - \nabla f(\mathbf{W}^{(r)})] \right\|_\mathrm{F}^2 \big]$.
% The key idea is to utilize Jensen's inequality (i.e., $\mathbb{E} \left\| \sum_{i=1}^n \bm{x}_i \right\|^2 \leq n \sum_{i=1}^n \mathbb{E} \left\| \bm{x}_i \right\|^2, \forall \bm{x}_i \in \mathbb{R}^{d}$) and $\|\mathbf{A}\mathbf{B}\mathbf{x}\|_\mathrm{F}^2 \leq \|\mathbf{A}\|_\mathrm{F}^2\|\mathbf{B}\|_\mathrm{F}^2\|\mathbf{x}\|_\mathrm{F}^2$ to characterize an upper bound of the variance. 
% The proof is concluded by using the properties in Lemma \ref{lem-1}.
% The detailed proof is deferred to the Appendix.
% \ref{proof-lem-3}.
Please refer to the Appendix.
\end{proof}

With Lemmas \ref{lem-2} and \ref{lem-3}, we establish the convergence of SCFL in the following theorem.
% \begin{theorem}\label{thm-convergence}
% Define the optimization gap after a duration of $T_\text{tot}$ as $G(T_\text{tot}) \triangleq f(\mathbf{W}^{(R)}) - f^*$ where $R=\left\lceil \frac{T_\text{tot}}{T} \right\rceil$ and $f^*=\min_{\mathbf{W}} {f(\mathbf{W})}$.
% With Assumption \ref{ass:norm}, if the learning rate is chosen as $\eta_r=\frac{\beta}{\gamma+r}$ for some $\beta > \frac{1}{\alpha}$ and $\gamma>0$ such that $\eta_1\leq \frac{1}{\zeta}$, then we have:
% % $\sum_{r=1}^{\infty} \eta_r =\infty$ and $\sum_{r=1}^{\infty} \eta_r^2 < \infty$,
% \begin{equation}
% % \begin{split}
%     G(T_\text{tot}) 
%     \leq \frac{\nu}{\gamma+R},
%     % (1-\eta_1\alpha)\Delta
%     % + \frac{\zeta}{2} \left(\sum_{r=1}^R \eta_r^2 \right) h(c, \sigma, b_\mathrm{s},\{b_i\}, T),
% % \end{split}
% \label{eq:gap}
% \end{equation}
% where $\nu=\max \left\{\frac{\beta^2\zeta h(c, \sigma, b_\mathrm{s},\{b_i\}, T)}{2(\beta\alpha-1)}, (\gamma+1)\Delta \right\}$,
% and $\Delta \triangleq f(\mathbf{W}^{(1)}) - f^*$ is the initial error.
% \end{theorem}
% \begin{proof}
% The proof follows Theorem 4.7 in \cite{bottou2018optimization} by combining Lemma \ref{lem-2} and Lemma \ref{lem-3}.
% %and is obtained by applying Lemma \ref{lem-2} and Lemma \ref{lem-3}.
% \end{proof}
\begin{theorem}\label{thm-convergence}
% Define $\Delta \triangleq f(\mathbf{W}^{(1)}) - f^*$ as the initial error, and 
Define the optimality gap after a duration of time $T_\text{tot}$ as $G(T_\text{tot}) \triangleq \mathbb{E}[ f( \frac{1}{R} \sum_{r=1}^R \mathbf{W}^{(r)})] - \min_{\mathbf{W}} {f(\mathbf{W})}$, where $R=\left\lceil \frac{T_\text{tot}}{T} \right\rceil$.
% where $f^* \triangleq \min_{\mathbf{W}} {f(\mathbf{W})}$.
With Assumption \ref{ass:norm}, if the learning rate is chosen as $\eta_r=\frac{1}{\zeta+\frac{1}{\gamma}}$ and $\gamma=\sqrt{\frac{4\phi^2}{\rho r}}$, we have:
% \vspace{-5pt}
%With Assumption \ref{ass:norm} and choosing the learning rate as $\eta_r=\frac{1}{\zeta+\frac{1}{\gamma}}$ and $\gamma=\sqrt{\frac{4\phi^2}{\rho r}}$, we have:
% $\sum_{r=1}^{\infty} \eta_r =\infty$ and $\sum_{r=1}^{\infty} \eta_r^2 < \infty$,
\begin{equation}
    G(T_\text{tot}) 
    \leq \sqrt{\frac{4\phi^2\rho}{R}} + \frac{2\phi^2\zeta}{R}.
\label{eq:gap}
\vspace{-1em}
\end{equation}
% where $R=\left\lceil \frac{T_\text{tot}}{T} \right\rceil$ denotes the number of training epochs.
% where $\nu=\max \left\{\frac{\beta^2\zeta h(c, \sigma, b_\mathrm{s},\{b_i\}, T)}{2(\beta\alpha-1)}, (\gamma+1)\Delta \right\}$.
\label{Theorem:convergence_bound}
% \vspace{-1em}
\end{theorem}
\begin{proof}
According to Assumption \ref{ass:norm}, the global loss function is $\zeta$-smooth and the model parameter $\mathbf{W}$ is bounded by $\sup_{\mathbf{W}^{(r)}} \| \mathbf{W}^{(r)} - \mathbf{W}^{(1)}\|_\mathrm{F}^2 = 2\phi^2$. 
% {\color{blue}Following Theorem 6.3 in \cite{convexbook}, we conclude the proof by utilizing the results in Lemmas \ref{lem-2}-\ref{lem-3} and the fact that $\min_r f(\mathbf{W}^{(r)}) \leq f( \frac{1}{R} \sum_{r=1}^R \mathbf{W}^{(r)})$.}
Then the result in (\ref{eq:gap}) is concluded by following Theorem 6.3 in \cite{convexbook} utilizing the results in Lemmas \ref{lem-2} and \ref{lem-3}.
\end{proof}

\begin{remark}
% To achieve $G_\varepsilon$-accuracy, i.e., $G(T_\text{tot})=G_\varepsilon$, the required training time is $T_\text{tot}= R T$ where $R = \frac{\nu}{G_\varepsilon} - \gamma$.
To achieve $G_\varepsilon$-accuracy in (\ref{eq:gap}), i.e., $G(T_\text{tot})=G_\varepsilon$, the required training time is {\color{black} $T_\text{tot}= \mathcal{O} \left(T \max( \frac{4\phi^2\rho}{G_\varepsilon^2}, \frac{2\phi^2\zeta}{G_\varepsilon} )\right) $}.
% \sout{In general, a larger variance $\rho$ results in more training epochs $R$ to converge and thus longer training time.}
% {\color{red}In general, a larger variance $\rho$ requires more training epochs $R$ until convergence and result in longer training time.}
% Under a fixed $T$, larger batch sizes ($b_\mathrm{s}$ and $\{b_i\}$'s) can reduce the mini-batch sampling variance (i.e., smaller $\frac{(c-b_\mathrm{s})}{c b_\mathrm{s}}$ and $\frac{l_i(l_i-b_i)}{b_i}$) but also require more time for computation, leading to lower probabilities $p_i(T,b_i)$.
% To improve the arriving probability, we can increase the waiting time $T$, which causes a contradiction in reducing the total training time.
% Therefore, $T$, $b_\mathrm{s}$ and $\{b_i\}$'s should be jointly optimized to maximize the training efficiency.
\end{remark}

% \begin{equation}
% \begin{split}
%     \min_{b_\mathrm{s},\{b_i\}, T} & \frac{\beta^2\zeta h(c, \sigma, b_\mathrm{s},\{b_i\}, T)}{2(\beta\alpha-1)} \\
%     \text{s.t.}\; & RT\leq T_\text{tot}, \\
%     & T = b_\mathrm{s} \gamma \\
%     & \text{Pr}[ t_i(b_i) ] = p_i(T,b_i), \forall i
% \end{split}
% \end{equation}
% smith2017bayesian: This optimal batch size is proportional to the learning rate and training set size

% \begin{remark}
% When the coded data received on the server is less noisy (i.e., smaller $\sigma^2$) or of a higher volume (i.e., larger $c$), the variance term $h(c, \sigma, b_\mathrm{s},\{b_i\}, T)$ is lower, leading to a smaller optimality gap within the given time $T_\text{tot}$.
% \end{remark}

\subsection{Privacy Analysis}

To characterize the privacy leakage caused by outsourcing the coded dataset $\tilde{\mathbf{X}}^{(i)}$, we adopt an $\epsilon$-mutual information differential privacy ($\epsilon$-MI-DP) metric defined as follows.
% \begin{definition}
% ($\epsilon$-MI-DP \cite{MI-DP}) A randomized mechanism $q(\cdot)$ satisfies $\epsilon$-mutual information differential privacy if
% \begin{equation}
%     \sup _{i, P\left(\mathcal{D}^{N}\right)} I\left(\mathcal{D}_{i} ; q\left(\mathcal{D}^{N}\right) \mid \mathcal{D}^{-i}\right) \leq \epsilon,
% \end{equation}
% where the supremum is taken over all distribution on dataset $\mathcal{D}^{N} = \{D_{1},\cdots,D_{N}\}$, and $\mathcal{D}^{-i}$ denotes the set of dataset  excluding $D_{i}$.
% \end{definition}
\begin{definition}
% i.e., $\tilde{\mathbf{X}}^{(i)}= q(\mathbf{X}^{(i)})$,
\textbf{($\epsilon$-MI-DP \cite{MI-DP})} A randomized mechanism $q(\cdot)$ that encodes local data $\mathbf{X}^{(i)}$ to $\tilde{\mathbf{X}}^{(i)}$ satisfies the $\epsilon$-mutual information differential privacy if
\begin{equation}
    {\color{black}\sup_{k, P_{\mathbf{X}^{(i)}}}} I\left(\mathbf{X}_{k}^{(i)} ; \tilde{\mathbf{X}}^{(i)} | \mathbf{X}_{-k}^{(i)}\right) \leq \epsilon_{i},
    \label{DP-def}
\end{equation}
where the supremum is taken over all distributions $P_{\mathbf{X}^{(i)}}$ of the local dataset $\mathbf{X}^{(i)}$, and $\mathbf{X}_{-k}^{(i)}$ denotes the dataset $\mathbf{X}^{(i)}$ excluding the $k$-th sample $\mathbf{X}_{k}^{(i)}$.
\end{definition}

Notably, a smaller value of the privacy budget $\epsilon_i$ in \eqref{DP-def} offers better privacy protection.
The following theorem gives the privacy budget $\epsilon_{i}$ when sharing the coded dataset $(\tilde{\mathbf{X}}^{(i)},\tilde{\mathbf{Y}}^{(i)})$.

\begin{theorem} \label{theorem1:privacy_budget}
% Suppose each data point $X_{k}^{(i)}$ is in an $\ell_{\infty}$-norm unit ball, the privacy budget is as follows\footnote{The requirement that the data points should be in an $\ell_{\infty}$-norm unit ball can be easily met by normalizing the dataset as $\frac{X^{(i)}}{\|X^{(i)}\|_{\infty}}$.}:
% \footnote{Although this requires that the data points should be in a $\ell_{\infty}$-norm unit ball, we can simply normalize the dataset by $\frac{X^{(i)}}{\|X^{(i)}\|_{\infty}}$ in practice to meet the requirement.}:
With Assumption \ref{ass:max_entry}, the privacy budget of client $i$ is given as follows:
% of $\tilde{\mathbf{X}}^{(i)}$
\begin{equation}
    \epsilon_{i} = \frac{1}{2} \log_{2} \bigg( 1+\frac{c}{h^{2}\big(\tilde{\mathbf{X}}^{(i)}\big)+ \sigma^2} \bigg),
\end{equation}
where $h(\tilde{\mathbf{X}}^{(i)}) \triangleq \min\limits_{k_{2}} \sqrt{\sum\limits_{k_{1}=1}^{l_{i}}|\mathbf{X}_{k_{1},k_{2}}^{(i)}|^{2} \!-\! \max\limits_{k_{3} \in[l_{i}]}|\mathbf{X}_{k_{3},k_{2}}^{(i)}|^{2}}$ with the $\mathbf{X}_{j,k}^{(i)}$ denoting the $(j,k)$-th entry of matrix $\mathbf{X}^{(i)}$.
% \sout{and $|\mathbf{X}_{j,k}^{(i)}|^{2}$ represents element-wise multiplication.}
In particular, we select $\epsilon \triangleq \max_{i} \epsilon_{i}$ as the privacy budget for coded data sharing.
%of the learning framework.

%$\epsilon \triangleq \max_{i=1,\dots,n} \epsilon_{i}$

% the privacy budget of SCFL is defined as $\epsilon \triangleq \max_{i=1,\dots,n} \epsilon_{i}$.

%$\mathbf{X}_{i,j}$

% i.e., the maximum privacy budget of clients.
% \begin{equation}
%     f\left(\tilde{\mathbf{X}}^{(i)}\right)=\min _{k_{2}} \sqrt{\sum_{k_{1}=1}^{l_{i}}\left|\mathbf{X}_{k_{1},k_{2}}^{(i)}\right|^{2}-\max _{k_{3} \in\left[l_{i}\right]}\left|\mathbf{X}_{k_{3},k_{2}}^{(i)}\right|^{2}}.
% \end{equation}
\end{theorem}
\begin{proof}
The proof is similar to that of Theorem 2 in \cite{Privacy_utility_tradeoff}.
\end{proof}

%By leveraging the result for random linear projection in \cite{15}, the privacy budget $\epsilon_{i} = \frac{1}{2} \log _{2}\left(1+\frac{\color{red}n^{\prime}}{f^{2}\left(\tilde{\mathbf{X}}^{(i)}\right)+\sigma^{2}}\right)$, where $f\left(\tilde{\mathbf{X}}^{(i)}\right)=\min _{k_{2} \in[q]} \sqrt{\sum_{k_{1}=1}^{l_{i}}\left|\mathbf{x}_{k_{1}}^{(i)}\left(k_{2}\right)\right|^{2}-\max _{k_{3} \in\left[l_{i}\right]}\left|\mathbf{x}_{k_{3}}^{(i)}\left(k_{2}\right)\right|^{2}}$.

\begin{remark}
The privacy budget in CodedFedL \cite{CFL_journal} can be viewed as a special case of Theorem \ref{theorem1:privacy_budget} with $\sigma=0$. 
By adding Gaussian noise to the coded data, the proposed SCFL provides better privacy protection than CodedFedL.
% {\color{red} more general}
\end{remark}

\begin{remark}\label{rm:tradeoff}
\textbf{(Privacy-performance tradeoff)}
\color{black}
According to Theorems \ref{thm-convergence} and \ref{theorem1:privacy_budget}, there is a tradeoff between privacy protection and convergence performance.
Particularly, increasing the coded data size $c$ or decreasing the additive noise level $\sigma$ leads to a smaller optimality gap, but it results in more severe privacy leakage.
%we prefer a larger $c$ and a smaller $\sigma^2$ (i.e., to share more coded data with less noise) in order to obtain a better model performance.
%However, it increases the leakage of privacy according to Theorem \ref{theorem1:privacy_budget}.
%According to Theorems \ref{thm-convergence} and \ref{theorem1:privacy_budget}, we prefer a larger $c$ and a smaller $\sigma^2$ (i.e., to share more coded data with less noise) in order to obtain a better model performance.
% However, it increases the leakage of privacy according to Theorem \ref{theorem1:privacy_budget}.
\end{remark}

% \begin{figure*}[!t]
% \setlength\abovecaptionskip{-10pt}
% \centering
% \subfigure{
% \begin{minipage}[t]{0.3\linewidth}
% \centering
% \includegraphics[width=1\linewidth]{GMM.eps}
% %\caption{fig1}
% \label{awgn+mnist}
% \end{minipage}%
% }%
% \subfigure{
% \begin{minipage}[t]{0.3\linewidth}
% \centering
% \includegraphics[width=1\linewidth]{MNIST.eps}
% %\caption{fig2}
% \label{bsc+mnist}
% \end{minipage}%
% }%
% \subfigure{
% \begin{minipage}[t]{0.3\linewidth}
% \centering
% \includegraphics[width=1\linewidth]{CIFAR.eps}
% %\caption{fig2}
% \label{awgn+cifar10}
% \end{minipage}
% }%
% \centering
% \caption{Convergence time of different FL methods on (a) the synthetic Gaussian dataset (IID), (b) the MNIST dataset (non-IID), and (c) the CIFAR-10 dataset (non-IID).}
% \label{Fig:convergence}
% % \vspace{-1em}
% \end{figure*}
% \begin{figure*}[t]
% \setlength\abovecaptionskip{-5pt}
% \setlength\belowcaptionskip{-10pt}
%     \centering
%     \subfigure
%     {\includegraphics[width=0.32\linewidth]{GMM.eps}}
%     \subfigure
%     {\includegraphics[width=0.32\linewidth]{MNIST.eps}}
%     \subfigure
%     {\includegraphics[width=0.32\linewidth]{CIFAR.eps}}
%     \subfigure
%     \caption{TBA.}
%     \label{Fig:convergence}
% \end{figure*}

\vspace{-5pt}
\section{Numerical Experiments}\label{sec:experiment}
\vspace{-3pt}
% {\color{blue}
In this section, we evaluate the performance of the proposed SCFL framework on two image classification tasks. 
% In this section, we demonstrate the performance improvements of the proposed SCFL framework via numerical experiments on three linear regression tasks. 
% }

\subsection{Experimental Setup}

\subsubsection{Wireless Edge Environment}
% {\color{blue}
We consider a wireless network with a server and $n \!=\! 20$ edge devices, using the delay model described in Section \ref{subsec:delay} to compute the overall training time.
% To simulate heterogeneity, t
The downlink data rate of each client is set to $r_{\text{downlink}} \!=\! 1$ Mbps, and the uplink data rate of device $i$ is $r_{\text{uplink},i} \!=\! \mu_{\text{uplink},i} \!\times\! 1$ Mbps, where $\mu_{\text{uplink},i}$ is sampled from a uniform distribution $U(0.3,1)$.
The transmission failure probability $p_{l_i}$ is assumed to be $0.1$.
% }
Besides, we randomly generate $20$ MAC rates, i.e., $\text{MACR}_{i} \!=\! \mu_{\text{comp},i} \!\times\! 1,536$ KMAC per second, where $\mu_{\text{comp},i}$ is sampled from a uniform distribution $U(0.1,1)$.
The computation rate of the server is set as $15,360$ KMAC per second \cite{CFL}.

\subsubsection{Baselines} We compare SCFL with the following baseline FL methods:
\begin{itemize}
\item \textbf{FL-PMA} ($\psi$): 
In the partial model aggregation (PMA) strategy \cite{PartialWorkerLogistic}, clients compute the stochastic gradient over the local mini-batches, and the server aggregates the gradients received from the first-arrived $(1-\psi)n$ clients. 
In particular, setting $\psi \!=\! 0$ corresponds to FedAvg \cite{fedavg} that aggregates all the gradients in each training round.

%waiting for all clients.
%the server waits for only an appropriate number of device models in each round
%We take partial model aggregation (PMA) into comparison
%In the standard federated learning setting, clients compute the stochastic gradient over the local mini-batches, and the server aggregates the gradients received from the first-arrived $(1-\psi)n$ clients. In particular, $\psi \!=\! 0$ corresponds to waiting for all clients.
%from all the clients.
%This corresponds to 
{\color{black}
\item \textbf{CFL-FB} \cite{CFL} and \textbf{CodedFedL} \cite{CFL_journal}: Before training starts, the clients generate coded datasets based on the weighted local datasets and share them with the server.
In each training epoch, the server and clients compute the gradients based on their data samples.
After waiting for a duration of time, the server aggregates the received gradients to update the global model.
Particularly, CodedFedL computes the stochastic gradients on a batch of samples, while the CFL-FB method trains model in a full-batch manner.
}
%}
% \item S-CodedFedL (Step) \cite{CFL_journal}: Stochastic CodedFedL
% Each client generates a parity dataset and upload it to the server for computing. In the training, clients compute gradients on the all data samples of remaining dataset. 
%, where the expected value of aggregate return is equal to the totality of raw data points.
%at the start of the training procedure, by taking linear combinations of features and labels in the local dataset.
%each client device privately generates parity training data and shares it with the central server only once at the start of the training phase. The central server can then preemptively perform redundant gradient computations on the composite parity data to compensate for the erased or delayed parameter updates, and achieves that the expected value of aggregate return is equal to the totality of raw data points spread across n edge devices.
\item \textbf{DP-CFL} \cite{anand2021differentially}:
% {\color{red}
Each client first generates a coded dataset by perturbed random linear combinations of its data for uploading to the server.
The server performs gradient descent based on the coded datasets with no further communication with the clients.
% server combine coded data generated from 
% computes the gradients on the coded datasets, which are generated by the clients in the same way as ours.
% }
% {\color{blue}The server computes the gradients on the coded datasets, which are generated by the clients in the same way as ours.}
\end{itemize}

\begin{figure}[!t]
\centering
\subfigure[MNIST dataset]{
\begin{minipage}[t]{0.48\linewidth}
\centering
\includegraphics[width=1\linewidth]{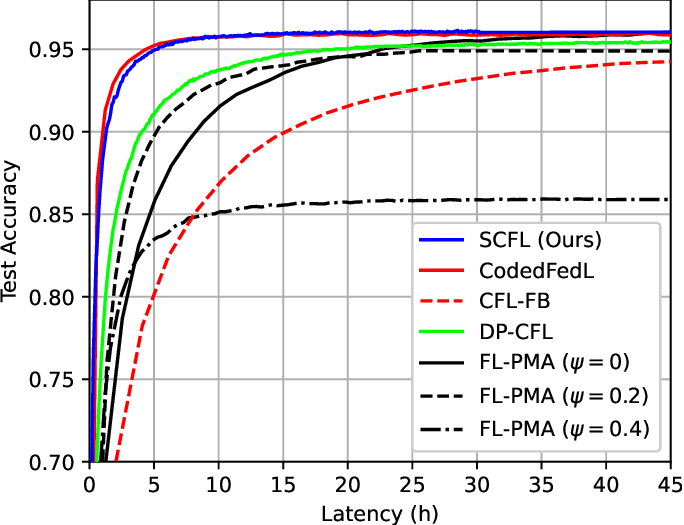}
%\caption{fig2}
\label{bsc+mnist}
\end{minipage}%
}%
\subfigure[CIFAR-10 dataset]{
\begin{minipage}[t]{0.48\linewidth}
\centering
\includegraphics[width=1\linewidth]{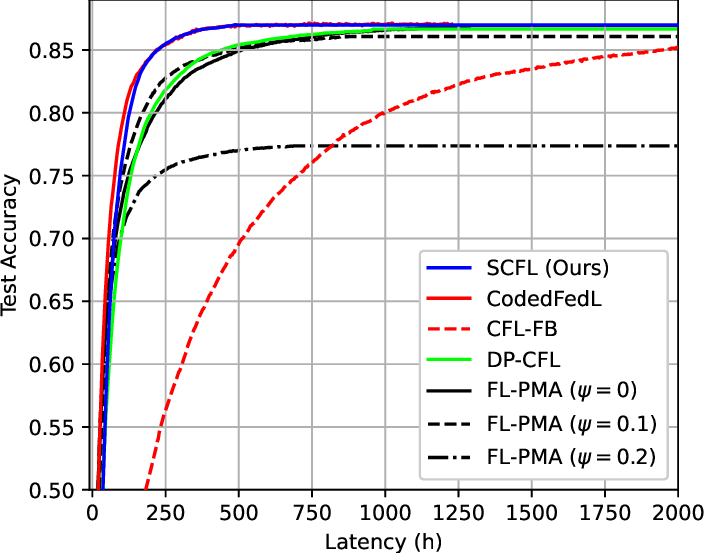}
%\caption{fig2}
\label{awgn+cifar10}
\end{minipage}
}%
\centering
\caption{Convergence time of different FL methods on (a) the MNIST dataset and (b) the CIFAR-10 dataset.}
\label{Fig:convergence}
\vspace{-1.1em}
\end{figure}

%In the Parallel SGD framework, the batch size of each device is set to 64, 256, 256 on the synthetic Gaussian dataset, MNIST, and CIFAR-10, respectively.
%In our method, we select the waiting time $T$ and the batch sizes $\{b_{i}\}$'s by minimizing the RHS of \eqref{eq:gap}.
%In particular, the CodedFedL method needs to find the appropriate sizes of the local datasets and composite dataset to satisfy the requirement that the expected value of the aggregate return is equal to the totality of raw data points\footnote{The optimal sizes are obtained by solving the optimization problem (23) in \cite{CFL_journal}.}.

\subsubsection{Dataset} We consider two benchmarking datasets, i.e., the MNIST \cite{mnist} and CIFAR-10 \cite{cifar10} datasets, to conduct the experiments.
{\color{black}
To simulate the non-IID data distribution, we adopt the skewed label partition \cite{hsieh2020non} to shuffle the MNIST and CIFAR-10 datasets.
Specifically, we sort a dataset by the labels, divide it into 20 shards with identical sizes, and assign one shard to each client.
Following \cite{CFL_journal}, we leverage the random Fourier feature mapping (RFFM) \cite{rahimi2008uniform} to transform the MNIST classification task into a linear regression problem. Each transformed vector has a size of $2,000$.
Besides, each image in CIFAR-10 is represented by a $4096$-dimensional feature vector extracted by a pretrained VGG model \cite{vgg}.}
% {\color{red}
We perform mini-batch SGD to achieve efficient federated learning, where each local dataset is partitioned into 30 subsets on MNIST and 20 subsets on CIFAR-10.
The CodedFedL method obtains the optimal client batch sizes and coded datasets by solving the optimization problem of (23) in \cite{CFL_journal}.
{\color{black}
For fair comparisons, we implement the same batch sizes for the SCFL, CodedFedL, DP-CFL, and FL-PMA methods.
%\sout{In particular, the F-CFL method performs gradient descent in a full-batch manner.}
%For fair comparisons, our SCFL method selects the same batch sizes as CodedFedL, and the DP-CFL follows the same size of the coded dataset. 
}
%For fair comparisons, we implement the same batch sizes for all the methods.

\subsection{Convergence Rate}
\label{Sec:Exp_Efficiency}

% {
% \color{red}
In this subsection, we compare the convergence rates of different methods by setting the additive noise level to zero.
The results in Fig. \ref{Fig:convergence} show that the CFL-FB method exhibits high training latency, which demonstrates the effectiveness of mini-batch sampling in speeding up convergence.
%CFL methods converge faster than the conventional FL scheme (i.e., FL-PMA ($\psi \!=\!0$)).
We also see that the conventional FL scheme (i.e., FL-PMA ($\psi \!=\!0$)) converges slower than the SGD-based CFL methods (i.e., SCFL, CodedFedL, and DP-CFL) due to the straggling effect.
%This is because the conventional FL scheme suffers from the straggling effect in federated learning.
%The results depicted in Fig. \ref{Fig:convergence} show that all the CFL methods converge faster than the conventional FL scheme (i.e., FL-PMA ($\psi \!=\!0$)).
Although FL-PMA can improve the convergence speed by dropping more stragglers (i.e., increasing the value of $\psi$), a larger dropout rate $\psi$ leads to more severe performance degradation especially in the non-IID scenario.
Besides, the DP-CFL method prolongs the training process compared with SCFL.
This is because DP-CFL restricts the gradient computation on the server without utilizing the clients' computational resources.
Moreover, the CodedFedL method has a comparable convergence rate to our method, but it does not provide an effective mechanism to adjust the privacy budget in coded data sharing.
We investigate the privacy-performance tradeoff in the next subsection.

\begin{figure}[!t]
\setlength\abovecaptionskip{0pt}
\centering
\subfigure[MNIST dataset]{
\begin{minipage}[t]{0.45\linewidth}
\centering
\includegraphics[width=1\linewidth]{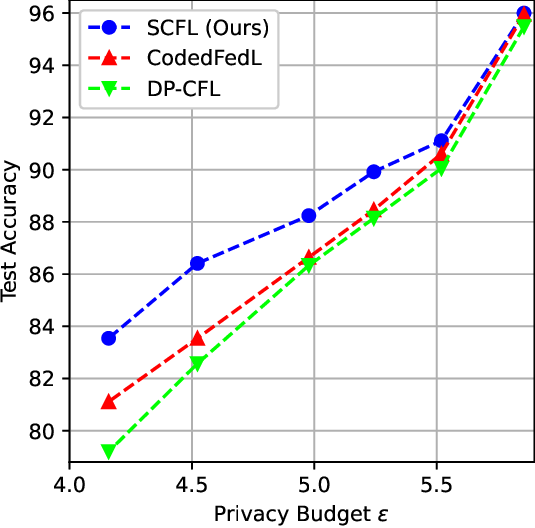}
%\caption{fig1}
\end{minipage}%
}%
\subfigure[CIFAR dataset]{
\begin{minipage}[t]{0.45\linewidth}
\centering
\includegraphics[width=1\linewidth]{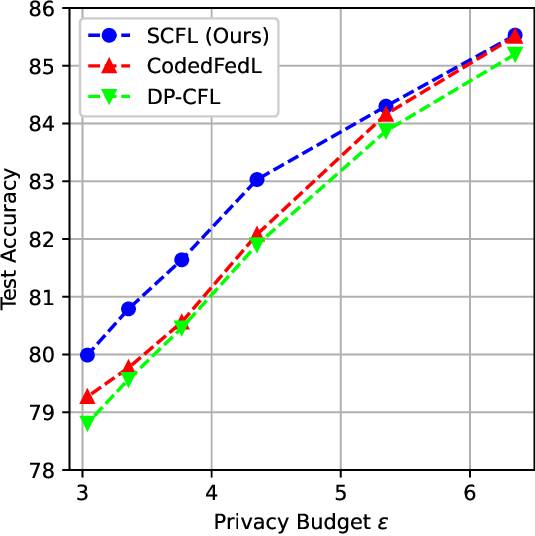}
%\caption{fig2}
\end{minipage}%
}%
\centering
\caption{The privacy-performance tradeoff on (a) the MNIST dataset and (b) the CIFAR-10 dataset.}
\label{Fig:privacy-performance tradeoff}
\vspace{-1.3em}
\end{figure}

\subsection{Privacy-Performance Tradeoff}
\vspace{-3pt}

We compare the learned model performance of SCFL, CodedFedL, and DP-CFL subject to different privacy budgets $\epsilon$, where $\epsilon$ is adjusted by varying the additive noise level.
%{\color{red} As latency}
As observed in Fig. \ref{Fig:privacy-performance tradeoff}, reducing the privacy budget (i.e., the privacy constraint becoming more restrictive) degrades the model performance, which is consistent with the analysis in Remark \ref{rm:tradeoff}.
% which is consistent with the analyses in Theorem \ref{Theorem:convergence_bound} and Remark \ref{rm:tradeoff}.
Besides, SCFL achieves a better privacy-performance tradeoff than other baseline CFL methods.
As the privacy budget reduces, larger additive noise leads to biased gradient estimates in the SCFL and CodedFedL methods.
%high biases in gradient estimates of the SCFL and CodedFedL methods.
In comparison, our gradient aggregation scheme in (\ref{eq:acc}) adds a make-up term to mitigate the variance in model updating.

\section{Conclusions}\label{sec:conclusion}
\vspace{-3pt}
In this paper, we proposed a novel algorithm to alleviate the straggler issue in federated learning, namely stochastic coded federated learning (SCFL).
SCFL enjoys high training efficiency without impairing model accuracy by adopting a mini-batch SGD algorithm.
We provided both the convergence and privacy analysis for SCFL, which showed a tradeoff between model performance and privacy.
% the deployment of SCFL.
Simulations verified this tradeoff and demonstrated that SCFL achieves fast convergence while preserving privacy.
For future works, it is worth investigating how to extend SCFL to other learning tasks.
%can help to adaptively control the training process.
% To improve communication efficiency, we adopt mini-batch SGD on clients. Then we theoretically analyzed the privacy guarantee and convergence speed. Based on the analysis, we propose an optimization problem to select the best batch size and time for each epoch. As for future works, we plan to conduct experiments to demonstrate the effectiveness of our proposed method. Besides, some intractable parameters in the optimization cause difficulty to find a closed-form solution in practice, which asks for a more effective algorithm.

% \section*{Appendix: Proof of Lemma \ref{lem-3}}{Proof of Lemma } \label{proof-lem-3}
\vspace{-5pt}
\appendix
\vspace{-3pt}
\section{Proof of Lemma \ref{lem-3}} \label{proof-lem-3}
Define the full-batch gradient on the coded dataset $(\tilde{\mathbf{X}},\tilde{\mathbf{Y}})$ as $\nabla f_\mathrm{s} (\mathbf{W}^{(r)}) \!\triangleq\! \frac{1}{c} \tilde{\mathbf{X}}^\TT (\tilde{\mathbf{X}} \mathbf{W}^{(r)} - \tilde{\mathbf{Y}})$. 
Given the independent error sources (i.e., mini-batch sampling in \eqref{eq:help-1} and data coding in \eqref{eq:help-2}), we decompose the variance caused by the server side as a summation of the following inequalities:
% \begin{align}
%     & \quad \; \mathbb{E} \bigg[ \left\| g_\mathrm{s}(\mathbf{W}^{(r)}) + g_\mathrm{o}(\mathbf{W}^{(r)}) - \nabla f(\mathbf{W}^{(r)}) \right\|_\mathrm{F}^2 \bigg] \nonumber\\
%     &= \mathbb{E} \bigg[\bigg\| g_\mathrm{s}(\mathbf{W}^{(r)}) - \nabla f_\mathrm{s} (\mathbf{W}^{(r)}) \nonumber\\
%     &\quad + \nabla f_\mathrm{s} (\mathbf{W}^{(r)}) + g_\mathrm{o}(\mathbf{W}^{(r)}) - \nabla f(\mathbf{W}^{(r)}) \bigg\|_\mathrm{F}^2 \bigg] \nonumber\\
%     &\overset{(b)}{=} \mathbb{E} \bigg[\left\| g_\mathrm{s}(\mathbf{W}^{(r)}) - \nabla f_\mathrm{s} (\mathbf{W}^{(r)}) \right\|_\mathrm{F}^2 \bigg] \nonumber \\
%     &\quad + \mathbb{E} \bigg[\left\| g_\mathrm{o}(\mathbf{W}^{(r)}) + \nabla f_\mathrm{s} (\mathbf{W}^{(r)}) - \nabla f(\mathbf{W}^{(r)}) \right\|_\mathrm{F}^2 \bigg],
% \end{align}
% where (b) holds since the error caused by two parts are independent. The RHS can be further bounded as follows:
\begin{align*}   
    & \quad \; \mathbb{E} \big[\big\| g_\mathrm{s}(\mathbf{W}^{(r)}) - \nabla f_\mathrm{s} (\mathbf{W}^{(r)}) \big\|_\mathrm{F}^2 \big] \nonumber\\
    % &= \mathbb{E} \bigg[\left\| \frac{1}{b_\mathrm{s}} \mathbf{\hat{X}}^{(r)\TT}_\mathrm{s} (\mathbf{\hat{X}}^{(r)}_\mathrm{s} \mathbf{W}^{(r)} - \mathbf{\hat{Y}}^{(r)}_\mathrm{s}) - \frac{1}{c} \tilde{\mathbf{X}}^\TT (\tilde{\mathbf{X}} \mathbf{W}^{(r)} - \mathbf{Y}_\mathrm{s}) \right\|_\mathrm{F}^2 \bigg] \nonumber\\
    % &= \mathbb{E} \bigg[\bigg\| \tilde{\mathbf{X}}^\TT \bigg( \frac{1}{b_\mathrm{s}} \mathbf{S}^{(r)\TT}_\mathrm{s} \mathbf{S}^{(r)}_\mathrm{s} - \frac{1}{c} \mathbf{I} \bigg) (\tilde{\mathbf{X}} \mathbf{W}^{(r)} - \mathbf{Y}_\mathrm{s}) \bigg\|_\mathrm{F}^2 \bigg] \nonumber\\
\end{align*}
\begin{align}
    &\overset{\text{(a)}}{\leq} \big\| \tilde{\mathbf{X}}^\TT \big\|_\mathrm{F}^2 \mathbb{E} \bigg[ \frac{1}{c^2} \bigg\| \frac{c}{b_\mathrm{s}} \mathbf{S}^{(r)}_\mathrm{s} - \mathbf{I} \bigg\|_\mathrm{F}^2 \bigg] \big\| \tilde{\mathbf{X}} \mathbf{W}^{(r)} - \tilde{\mathbf{Y}} \big\|_\mathrm{F}^2 \nonumber\\
    &\overset{\text{(b)}}{\leq} \frac{c-b_\mathrm{s}}{c b_\mathrm{s}} \big\| \tilde{\mathbf{X}}^\TT \big\|_\mathrm{F}^2 \big\| \tilde{\mathbf{X}} \mathbf{W}^{(r)} - \tilde{\mathbf{Y}} \big\|_\mathrm{F}^2 \nonumber\\
    & \leq \frac{c-b_\mathrm{s}}{c b_\mathrm{s}} \zeta\kappa,
\label{eq:help-1}
\vspace{-5pt}
\end{align}
and
\begin{align}
\vspace{-5pt}
    & \quad\; \mathbb{E} \big[\big\|  \nabla f_\mathrm{s} (\mathbf{W}^{(r)}) + g_\mathrm{o}(\mathbf{W}^{(r)}) - \nabla f(\mathbf{W}^{(r)}) \big\|_\mathrm{F}^2 \big] \nonumber\\
    % &= \mathbb{E} \big[\big\| \mathbf{X}^\TT  \big(\frac{1}{c}\mathbf{G}^\TT\mathbf{G}-\mathbf{I}\big)  (\mathbf{X} \mathbf{W}^{(r)} - \mathbf{Y}) + \sigma \mathbf{X}^\TT \mathbf{G}^\TT \mathbf{N} \mathbf{W}^{(r)} \nonumber\\
    %  &\quad + \sigma \mathbf{N}^\TT \mathbf{G} (\mathbf{X} \mathbf{W}^{(r)} - \mathbf{Y}) + \sigma^2 \big( \mathbf{N}^\TT \mathbf{N} - \mathbf{I} \big) \mathbf{W}^{(r)} \big\|_\mathrm{F}^2 \big] \nonumber\\
    &\overset{\text{(c)}}{\leq} 4 \left\| \mathbf{X}^\TT \right\|_\mathrm{F}^2 \mathbb{E} \bigg[ \bigg\| \left(\frac{1}{c}\mathbf{G}^\TT\mathbf{G}-\mathbf{I}\right) \bigg\|_\mathrm{F}^2 \bigg] \left\| \mathbf{X} \mathbf{W}^{(r)} - \mathbf{Y} \right\|_\mathrm{F}^2 \nonumber\\
    & \quad + 4 \sigma^4 \mathbb{E} \bigg[ \bigg\| \bigg(\frac{1}{c}(\sum_{i=1}^n \mathbf{N}_i)^\TT(\sum_{i=1}^n \mathbf{N}_i)-n\mathbf{I}\bigg) \bigg\|_\mathrm{F}^2 \bigg] \big\| \mathbf{W}^{(r)} \big\|_\mathrm{F}^2 \nonumber\\
    & \quad + \frac{4\sigma^2}{c^2} \left\| \mathbf{X}^\TT \right\|_\mathrm{F}^2 \mathbb{E} \left[ \left\| \mathbf{G}^\TT \mathbf{N} \right\|_\mathrm{F}^2 \right] \left\| \mathbf{W}^{(r)\TT} \right\|_\mathrm{F}^2 \nonumber \\
    &\quad + \frac{4\sigma^2}{c^2} \mathbb{E} \left[ \left\| \mathbf{N}^\TT \mathbf{G} \right\|_\mathrm{F}^2 \right] \big\| \mathbf{X} \mathbf{W}^{(r)} \!-\! \mathbf{Y} \big\|_\mathrm{F}^2 \nonumber\\
    &\overset{\text{(d)}}{\leq} \frac{4}{c} (m+m^2) \zeta\kappa + \frac{4}{c} (d+d^2)n \sigma^4 \phi^2 \nonumber\\
    & \quad + \frac{4\sigma^2}{c^2} dmn \zeta\phi^2 + \frac{4\sigma^2}{c^2} dmn \kappa \label{eq:help-2}\\
    &= \frac{4}{c} [ (m+m^2) \zeta\kappa + (d+d^2)n \sigma^4 \phi^2 ] + \frac{4 dmn\sigma^2}{c^2} (\zeta\phi^2 + \kappa), \nonumber
\end{align}
% https://en.wikipedia.org/wiki/Distribution_of_the_product_of_two_random_variables
where (a) and (c) follow the inequality $\|\mathbf{A}\mathbf{B}\mathbf{x}\|_\mathrm{F}^2 \leq \|\mathbf{A}\|_\mathrm{F}^2\|\mathbf{B}\|_\mathrm{F}^2\|\bm{x}\|_2^2$ for any compatible matrices $\mathbf{A},\mathbf{B}$ and vector $\bm{x}$. (b) and (d) hold due to Lemma \ref{lem-1} and the fact that the expected values of $\left\| \mathbf{N}^\TT \mathbf{G} \right\|_\mathrm{F}^2$ and $\left\| \mathbf{G}^\TT \mathbf{N} \right\|_\mathrm{F}^2$ equal $dmn$. % = \|\mathbf{A}\|_\mathrm{F}^2\|\mathbf{x}\|_\mathrm{F}^2\|\mathbf{B}\|_\mathrm{F}^2
Defining the full-batch gradient on client $i$ as $\nabla f_i (\mathbf{W}^{(r)}) \!\triangleq\! \mathbf{X}^{(i)\TT} (\mathbf{X}^{(i)} \mathbf{W}^{(r)} - \mathbf{Y}^{(i)})$, we characterize the client-side variance as follows:
% }
\begin{align}
\vspace{-5pt}
    & \quad \; \mathbb{E} \bigg[\bigg\| \sum_{i=1}^{n} \frac{g_i (\mathbf{W}^{(r)})}{p_i(T,b_i)} \mathbbm{1}\{ T_i(b_i)\leq t \} - \nabla f(\mathbf{W}^{(r)}) \bigg\|_\mathrm{F}^2 \bigg] \nonumber\\
    % &= \mathbb{E} \bigg[\bigg\| \sum_{i=1}^{n} \bigg[ \frac{ \mathbbm{1} \{ T_i(b_i)\leq t \} - p_i(T,b_i)}{p_i(T,b_i)} g_i (\mathbf{W}^{(r)}) \nonumber\\
    % & \quad + g_i (\mathbf{W}^{(r)}) \big] - \nabla f(\mathbf{W}^{(r)}) \big\|_\mathrm{F}^2 \big] \nonumber\\
    &\overset{\text{(e)}}{\leq} 2 \mathbb{E} \bigg[ \left\| \sum_{i=1}^{n} \frac{ \mathbbm{1} \{ T_i(b_i)\leq t \} - p_i(T,b_i)}{p_i(T,b_i)} g_i(\mathbf{W}^{(r)}) \right\|_\mathrm{F}^2 \bigg] \nonumber\\
    & \quad +2 \mathbb{E} \bigg[\bigg\| \sum_{i=1}^{n} g_i (\mathbf{W}^{(r)}) - \nabla f (\mathbf{W}^{(r)}) \bigg\|_\mathrm{F}^2 \bigg] \nonumber\\
    % &= 2 \mathbb{E} \bigg[ \left\| \sum_{i=1}^{n} \frac{ \mathbbm{1} \{ T_i(b_i)\leq t \} - p_i(T,b_i)}{p_i(T,b_i)} g_i(\mathbf{W}^{(r)}) \right\|_\mathrm{F}^2 \bigg] \nonumber\\
    % & \quad +2 \mathbb{E} \bigg[\bigg\| \sum_{i=1}^{n} g_i (\mathbf{W}^{(r)}) - \sum_{i=1}^n\nabla f_i (\mathbf{W}^{(r)}) \bigg\|_\mathrm{F}^2 \bigg] \nonumber\\
    &\overset{\text{(f)}}{=} 2 \mathbb{E} \bigg[ \sum_{i=1}^{n} \frac{ \mathbb{E} [\mathbbm{1} \{ T_i(b_i)\leq t \} - p_i(T,b_i)]^2}{p_i^2(t,b_i)} \left\| g_i(\mathbf{W}^{(r)}) \right\|_\mathrm{F}^2 \bigg] \nonumber\\
    & \quad + 2 \mathbb{E} \bigg[\bigg\| \mathbf{X}^{(i)\TT} \bigg( \sum_{i=1}^n \frac{l_i}{b_i} \mathbf{S}^{(i,r)}-\mathbf{I} \bigg)  (\mathbf{X}^{(i)} \mathbf{W}^{(r)} - \mathbf{Y}^{(i)}) \bigg\|_\mathrm{F}^2 \bigg] \nonumber\\
    % &\leq 2 \mathbb{E} \bigg[ \sum_{i=1}^{n} \frac{1 - p_i(T,b_i)}{p_i(T,b_i)} \left\| g_i(\mathbf{W}^{(r)}) \right\|_\mathrm{F}^2 \bigg] \nonumber\\
    % & + 2 \mathbb{E} \bigg[\bigg\| \mathbf{X}^{(i)\TT} \left( \sum_{i=1}^n \frac{l_i}{b_i} \mathbf{S}^{(i,r)\TT} \mathbf{S}^{(i,r)}-\mathbf{I} \right)  (\mathbf{X}^{(i)} \mathbf{W}^{(r)} - \mathbf{Y}) \bigg\|_\mathrm{F}^2 \bigg] \nonumber\\
    % &\overset{()}{\leq} 2 \sum_{i=1}^{n} \frac{1 - p_i(T,b_i)}{p_i(T,b_i)} \zeta_i^2\kappa_i^2 \nonumber\\
    % & +2 \left\| \mathbf{X}^{(i)\TT} \right\|_\mathrm{F}^2 \left\| \mathbf{X}^{(i)} \mathbf{W}^{(r)} - \mathbf{Y}^{(i)} \right\|_2^2  \mathbb{E} \bigg[\bigg\| \sum_{i=1}^n \frac{l_i}{b_i} \mathbf{S}^{(i,r)}-\mathbf{I} \bigg\|_\mathrm{F}^2 \bigg] \nonumber\\
    &\overset{\text{(g)}}{\leq} 2 \sum_{i=1}^{n} \frac{1 - p_i(T,b_i)}{p_i(T,b_i)} \zeta_i^2\kappa_i^2 + 
    2 \sum_{i=1}^n \frac{l_i(l_i-b_i)}{b_i} \zeta_i^2\kappa_i^2,
    \vspace{-5pt}
\label{eq:help-3}
\end{align}
where (e) follows the Jensen's inequality, (f) holds since $\sum_{i=1}^n {\mathbf{Z}^{(i)\TT}}\mathbf{Z}^{(i)} = \mathbf{I}_m$ and $\sum_{i=1}^n \nabla f_i (\mathbf{W}^{(r)}) = \nabla f (\mathbf{W}^{(r)})$.
The first term in (g) follows $\mathbb{E} [\mathbbm{1} \{ T_i(b_i)\leq t \} - p_i(T,b_i)]^2 = p_i(T,b_i) (1 - p_i(T,b_i))$.
Besides, the proof for the second term in (g) is similar to \eqref{eq:help-1} and thus omitted.
% according to Lemma \ref{lem-1}, (c) and (d) hold. 
Then by summing up \eqref{eq:help-1}-\eqref{eq:help-3} we conclude the proof.
% into the following inequality we conclude the proof:
% \begin{equation}
% \begin{split}
%     &\quad \; \mathbb{E} \bigg[\left\| g(\mathbf{W}^{(r)}) - \nabla f(\mathbf{W}^{(r)}) \right\|_\mathrm{F}^2 \bigg] \\
%     &= \mathbb{E} \bigg[\bigg\| \frac{1}{2} \bigg[g_\mathrm{s} (\mathbf{W}^{(r)}) 
%     + \sum_{i=1}^n \frac{g_i (\mathbf{W}^{(r)})}{p_i(T,b_i)}  \mathbbm{1} \{ T_i(b_i)\leq t \} \bigg] \\
%     &\quad + g_\mathrm{o} (\mathbf{W}^{(r)}) - \nabla f(\mathbf{W}^{(r)}) \bigg\|_\mathrm{F}^2 \bigg]\\
%     &\leq \frac{1}{4} \mathbb{E} \bigg[\left\| g_\mathrm{s}(\mathbf{W}^{(r)}) + g_\mathrm{o} (\mathbf{W}^{(r)}) - \nabla f(\mathbf{W}^{(r)}) \right\|_\mathrm{F}^2 \bigg] \\
%     &\quad + \frac{1}{4} \mathbb{E} \bigg[\left\| \sum_{i=1}^{n} \frac{g_i (\mathbf{W}^{(r)})}{p_i(T,b_i)}  \mathbbm{1} \{ T_i(b_i)\leq t \} - \nabla f(\mathbf{W}^{(r)})] \right\|_\mathrm{F}^2 \bigg]
% \end{split}
% \end{equation}
% \end{proof}

% {\color{red} I didn't find a good alternative to Ref[16]. }
\bibliographystyle{./IEEEtran}
\bibliography{IEEEabrv,ref}

\end{document}